\documentclass[runningheads]{llncs}
\usepackage{graphicx}
\usepackage{microtype}
\usepackage{graphicx}
\usepackage{subfigure}
\usepackage{booktabs} 

\usepackage{amsmath}
\usepackage{subfigure}
\usepackage{amsfonts}
\usepackage{multirow}
\usepackage{bm}

\def\eqref#1{equation~\ref{#1}}

\def\1{\bm{1}}

\def\vx{{\bm{x}}}

\DeclareMathAlphabet{\mathsfit}{\encodingdefault}{\sfdefault}{m}{sl}
\SetMathAlphabet{\mathsfit}{bold}{\encodingdefault}{\sfdefault}{bx}{n}

\def\gA{{\mathcal{A}}}

\def\gE{{\mathcal{E}}}

\def\gL{{\mathcal{L}}}

\def\gR{{\mathcal{R}}}
\def\gS{{\mathcal{S}}}
\def\gT{{\mathcal{T}}}

\def\sA{{\mathbb{A}}}

\def\sD{{\mathbb{D}}}

\def\sS{{\mathbb{S}}}

\newcommand{\E}{\mathbb{E}}

\newcommand{\R}{\mathbb{R}}

\newcommand{\KL}{D_{\mathrm{KL}}}
\newcommand{\Var}{\mathrm{Var}}

\newcommand{\Cov}{\mathrm{Cov}}

\usepackage{pgfplots, pgfplotstable}
\pgfplotsset{compat=1.14}

\newcommand{\Bias}{\mathrm{Bias}}

\usepackage[misc]{ifsym}
\usepackage[hyphens]{url}

\usepackage{seqsplit}
\usepackage{tabularx}
\usepackage{xltabular}

\usepackage{fancyhdr}
\fancypagestyle{specialfooter}{%
	\fancyhf{}
	
	\fancyfoot[C]{Preprint manuscript accepted for publication at ECML-PKDD 2021.}
}

\begin{document}
	\title{Ensemble and Auxiliary Tasks for Data-Efficient Deep Reinforcement Learning}
	\titlerunning{Ensemble and Auxiliary Tasks for Data-Efficient Deep RL}
	%
	\author{Muhammad Rizki Maulana\orcidID{0000-0002-3457-2563} (\Letter) \and \\
		Wee Sun Lee\orcidID{0000-0002-0988-2500}}
	\authorrunning{M. R. Maulana \& W. S. Lee}
	%
	\institute{
		School of Computing, National University of Singapore, Singapore\\
		\email{rizki@u.nus.edu, leews@comp.nus.edu.sg}
	}
	%
	
	\toctitle{Ensemble and Auxiliary Tasks for Data-Efficient Deep Reinforcement Learning}
	\tocauthor{Muhammad~Rizki~Maulana, Wee~Sun~Lee}
	
	\maketitle              
	\begin{abstract}
		Ensemble and auxiliary tasks are both well known to improve the performance of machine learning models when data is limited. However, the interaction between these two methods is not well studied, particularly in the context of deep reinforcement learning. In this paper, we study the effects of ensemble and auxiliary tasks when combined with the deep Q-learning algorithm. We perform a case study on ATARI games under limited data constraint. Moreover, we derive a refined bias-variance-covariance decomposition to analyze the different ways of learning ensembles and using auxiliary tasks, and use the analysis to help provide some understanding of the case study. Our code is open source and available at \url{https://github.com/NUS-LID/RENAULT}. 
		
	\end{abstract}
	%
	%
	%
	
	\thispagestyle{specialfooter}
	\section{Introduction}
	Ensemble learning is a powerful technique to improve the performance of machine learning models on a diverse set of problems. 
	In reinforcement learning (RL), ensembles are mostly used to stabilize learning and reduce variability \cite{anschel2017averaged,chua2018deep,kurutach2018model}, and in few cases, to enable exploration \cite{osband2016deep,chen2017ucb}. Orthogonal to the utilization of ensembles, auxiliary tasks also enjoy widespread use in RL to aid learning \cite{jaderberg2016reinforcement,laskin2020curl,mirowski2016learning,kartal2019terminal}. 
	
	The interplay between these two methods has been studied within a limited capacity in the context of simple neural networks \cite{ye2006improving} and decision trees \cite{wang2008mtforest}. In reinforcement learning, this interaction -- to the best of our knowledge -- has not been studied at all. 
	
	Our principal aim in this work is to study ensembles and auxiliary tasks in the context of deep Q-learning algorithm. Specifically, we apply ensemble learning on the well-established Rainbow agent \cite{hessel2018rainbow,van2019use}, and additionally augment it with auxiliary tasks. 
	We study the problem theoretically through the use of bias-variance-covariance analysis and by performing an empirical case study on a popular reinforcement learning benchmark, ATARI games, under the constraint of low number of interactions \cite{kaiser2019model}. ATARI games offer a suite of diverse problems, improving the generality of the results, and data scarcity makes the effect of ensembles more pronounced. Moreover, under the constraint of low data, it is naturally preferable to trade-off the extra computational requirement of using an ensemble for the performance gain that it provides. 
	
	We derive a more refined analysis of the bias-variance-covariance decomposition for ensembles. The usual analysis assumes that each member of the ensemble is trained on a different dataset. Instead, we focus our analysis on a single dataset, used with multiple instantiations of a randomized learning algorithm. This is commonly how ensembles are actually used in practice; in fact, the multiple datasets that are used for training members of the ensemble are often constructed from a single dataset through the use of randomization.  Additionally, we introduce some new ``weak" auxiliary tasks that provide small improvements, based on model learning and learning properties of objects and events. We show how ensembles can be used for combining multiple "weak" auxiliary tasks to provide stronger improvements. 
	
	Our case study and analysis shows that,
	\begin{itemize}
		
		\item Independent training of ensemble members works well. Joint training of the entire ensemble reduces Q-learning error but, surprisingly did not perform as well as an independently trained ensemble.
		
		\item The new auxiliary tasks are ``weakly" helpful. Combining them together using an ensemble can provide a significant performance boost. We observe reduction in variance and covariance 
		with the use of auxiliary tasks in the ensemble.
		
		\item Despite their benefits, using all auxiliary tasks on each predictor in ensemble may results in poorer performance. Analysis indicates that this could cause higher bias and covariance due to loss of diversity. 
		
	\end{itemize}
	
	It is interesting to note that our ensemble, despite its simplicity, achieves better performance on 13 out of 26 games compared to recent previous works. Moreover, our ensemble with auxiliary tasks achieves significantly better human mean and median normalized performance; $1.6\times$ and $1.55\times$ better than data-efficient Rainbow \cite{van2019use}, respectively.
	
	\section{Related Works}
	
	\noindent
	\textbf{Reinforcement Learning and auxiliary tasks. } Rainbow DQN \cite{hessel2018rainbow} combines multiple important advances in DQN \cite{mnih2015human} such as learning value distribution \cite{bellemare2017distributional} and prioritizing experience \cite{schaul2015prioritized} to improve performance. Other works try to augment RL by devising useful auxiliary tasks. \cite{jaderberg2016reinforcement} proposed auxiliary tasks in the form of reward prediction as a classification of positive, negative, or neutral reward. \cite{jaderberg2016reinforcement} also proposed to predict changing pixels for the downsampled image. Recently, \cite{laskin2020curl} proposed the use of contrastive loss to learn better representation. 
	Other auxiliary tasks have been explored as well, such as depth prediction \cite{mirowski2016learning} and terminal prediction \cite{kartal2019terminal}. These auxiliary tasks are less general; they require domain with 3D inputs and problem with episodic nature, respectively. Although much research has been done with auxiliary tasks in RL, to the best of our knowledge, none of them investigated the use of auxiliary tasks in the context of ensemble RL.
	
	\noindent
	\textbf{Ensemble in Reinforcement Learning. } Ensemble methods have been explored in RL for various purposes \cite{anschel2017averaged,osband2016deep,chua2018deep,kurutach2018model}. \cite{anschel2017averaged} investigated the effect of ensemble in RL, especially pertaining to the reduction of target approximation error. In the model-based RL, \cite{chua2018deep} used ensemble to reduce modelling errors, and \cite{kurutach2018model} accelerated policy learning by generating experiences through ensemble of dynamic models. In the context of policy gradients, \cite{fujimoto2018addressing} utilized ensemble value function as a critique to reduce function approximation error. \cite{osband2016deep} proposed the use of ensemble for exploration by training an ensemble based on bootstrap with random initialization and randomly sampled policy from the ensemble. \cite{chen2017ucb} extended the idea and replaced the policy sampling with UCB. Finally, \cite{lee2020sunrise} proposed to combine ensemble bootstrap with random initialization \cite{osband2016deep}, weighted Bellman backup, and UCB \cite{chen2017ucb}. While they also studied the ensemble in the similar context, they did not attempt to explain the gain afforded by the ensemble, nor did they studied the effect of combining ensemble with auxiliary tasks. 
	
	\section{Background}
	
	\subsection{Markov Decision Process and RL}
	A sequential decision problem is often modeled as a Markov Decision Process (MDP). An MDP is defined with a 5-tuple $<\sS,\sA,R,T,\gamma>$ where $\sS$ and $\sA$ denote the set of states and actions, $R$ and $T$ represent the reward and transition functions, and $\gamma \in [0,1)$ is a discount factor of the MDP. Reinforcement Learning aims to find an optimal solution of a decision problem of unknown MDP. One of the well known model-free RL algorithms is Deep Q Learning (DQN) \cite{mnih2015human}, which learns a state-action ($s$,$a$) value function $Q(s,a;\theta)$ with neural networks parameterized by $\theta$. The Q-function is used to select action when interacting with environment; of which the experience is accumulated in the replay buffer $\sD$ for learning. 
	We refer the reader to Appendix \ref{appendix:background} for more details about MDP and RL.
	
	\subsection{Rainbow Agent}
	The Rainbow agent \cite{hessel2018rainbow} extends the DQN by introducing various advances in reinforcement learning. It uses Double-DQN \cite{van2016deep} to minimize the overestimation error. Instead of sampling uniformly from the replay buffer $\sD$, it assigns priority to each instance based on the temporal difference (TD) error \cite{schaul2015prioritized}. Its architecture decomposes advantage function from value function $Q(s,a) = V(s) + A(s,a)$ and learn them in an end-to-end manner \cite{wang2016dueling}. Moreover, categorical value distribution is learned in place of the expected state-action value function \cite{bellemare2017distributional}. 
	Thus, the loss function is given as follows.
	We denote the scalar-valued Q-function corresponding to the the distributional Q-function as $\hat{Q}$ for simplicity.
	\begin{align}
	\gL(\theta) &= \E\big[\KL[g_{s,a,r,s'}||Q(s,a;\theta)]\big]\\ \label{eq/Q-loss-distributional}
	g_{s,a,r,s'} &= \Phi_{Q(s',\hat{a}';\theta')} \Big(r + \gamma \gS \Big) \\
	\hat{a}' &= \arg \max_{a'} \hat{Q}(s',a';\theta)
	\end{align}
	where $\Phi_{Q(s',\hat{a}';\theta')}$ denotes a distributional projection \cite{bellemare2017distributional} based on categorical atom probabilities given by $Q(s',\hat{a}';\theta')$ for support $\gS$. $Q$ returns a column vector Softmax output with $|\gS|$ rows instead of a scalar value and $\gS$ is a column vector support of the categorical distribution. The scalar-valued Q function is computed by $\hat{Q}(s,a) = \gS^T Q(s,a)$. We refer the reader to the original paper \cite{bellemare2017distributional} for a detailed explanation.
	
	Multi-step returns is also employed to achieve faster convergence \cite{sutton2011reinforcement}. To aid exploration, NoisyNets \cite{fortunato2018noisy} is utilized; it works by perturbing the parameter space of the Q-function by injecting learnable Gaussian noise.
	
	Recently, \cite{van2019use} proposes a set of hyperparameters for Rainbow that works well on ATARI games under 100K interactions \cite{kaiser2019model}.

	\section{Rainbow Ensemble} \label{rainbow-ensemble}
	
	Several forms of ensemble agents have been proposed in the literature \cite{anschel2017averaged,osband2016deep,lee2020sunrise} with different variations of complexities. For ease of analysis, we propose to use a simple ensemble similar to Ensemble DQN \cite{anschel2017averaged}. The original Ensemble DQN was not combined with the recent advances in DQN such as distributional value function \cite{bellemare2017distributional}, prioritized experience replay \cite{schaul2015prioritized}, and NoisyNets \cite{fortunato2018noisy}. Here, we describe our ensemble, that we call \textit{REN} (Rainbow ENsemble), which combines a simple ensemble approach with modern DQN advances in Rainbow.
	
	REN is based on the following simple ensemble estimator:
	\begin{align}
	Q_{ens}^{(M)}(s,a) = \sum_{m=1}^M \frac{1}{M} Q_m(s,a;\theta_m) \label{eq/Ensemble*}
	\end{align}
	where $\theta_m$ is the parameter for the $m$-th Q function. It outputs a distributional estimate of the Q-function. The scalar-valued ensemble Q-function is given by $\hat{Q}_{ens}^{(M)}(s,a) = \gS^T Q_{ens}^{(M)}(s,a)$. We use this Q-function as our policy by taking the action that maximizes this state-action function,
	\begin{align}
	a &= \arg \max_a \hat{Q}_{ens}^{(M)}(s,a)
	\end{align}
	The ensemble distributional Q-function is used to compute the TD target $g_{s,a,r,s'}$ following \eqref{eq/Q-loss-distributional}, thus allowing reduction in the target approximation error (TAE) due to averaging \cite{anschel2017averaged}. The agent learns the estimator by optimizing the following loss,
	\begin{align}
	\gL(\theta_m) &= \E_{D_m}\big[ \KL[g_{s,a,r,s'}||Q(s,a;\theta_m) ] \big] \label{eq/Ensemble*-loss}
	\end{align}
	where $m$ is the index of the member of the ensemble and $\sD_m$ is the dataset (buffer) for the $m$-th member. The loss is computed for each member $m$ and optimized independently.
	
	Rainbow uses prioritized experience replay which assigns priority to each transition. This allows for important transitions -- transitions with higher error -- to be sampled more frequently. Since REN consists of multiple members, we adopt the use of multiple prioritized experience replay for each member; i.e., for each member $m$, we have a prioritized replay $\sD_m$ with priority updated following the loss $\gL(\theta_m)$. This allows each member to individually adjust their priority based on their individual errors, thus potentially enabling the ensemble to perform better. 
	
	Finally, NoisyNets is used for exploration in Rainbow. It takes the form of stochastic layers in the Q-function; each stochastic layer samples its parameters from a distribution of parameters modeled as a learnable Gaussian distribution. In our case, we have $M$ Q-functions, with each containing the stochastic layers of NoisyNets. 
	
	\section{Auxiliary Tasks for Ensemble RL}
	
	Auxiliary tasks have often been shown to improve performance in learning problems. However, the combination of auxiliary tasks and ensembles has not been extensively studied, particularly in the context of reinforcement learning.
	
	We use the following framework, where each member of the ensemble can be trained together with a different set of auxiliary tasks, for combining ensembles with auxiliary tasks. Let $\gT_m = \{t_{m,1}, ..., t_{m,N_m}\}$ be the set of auxiliary tasks for member $m$, we seek to optimize the following loss,
	\begin{align}
	\gL_A(\theta_m) &= \gL(\theta_m) + \E\Big[ \sum_{n=1}^{N_m} \alpha_{m,n}  \gL_{t_{m,n}}(\theta_m) \Big] \label{eq/aux-loss}
	\end{align}
	where $\alpha_{m,n}$ is the strength parameter and $\gL_{t_{m,n}}(\theta_m)$ is the auxiliary task's specific loss for task $t_{m,n}$. Each task may additionally includes a set of parameters that will be optimized jointly with the parameters of the member.
	
	Some questions immediately arise. Should every auxiliary tasks be used with every member of the ensemble? The other extreme would be using a single distinct auxiliary task with each member of the ensemble. If each auxiliary task is \emph{weak} in the sense of only providing a small improvement, can they be combined in the ensemble to provide much stronger improvements? We examine some of these questions in our analysis and case study.
	
	The framework can be viewed as the generalization of MTLE \cite{ye2006improving} and MTForest \cite{wang2008mtforest}, where each member of the ensemble is trained with the auxiliary task of predicting the value of a distinct component of the input vector. An instantiation of this framework with REN as the ensemble is denoted as \textit{RENAULT} (Rainbow ENsemble with AUxiLiary Tasks).
	
	We propose to use model learning (i.e., learning transition function and reward function) and learning to predict properties related to objects and events as our auxiliary tasks. Model learning has already been used in model-based RL \cite{kaiser2019model}, whereas predicting properties related to objects and events appears to be quite natural -- rewards and hence the returns are often associated with events in the environment. 
	
	We only consider tasks that can easily be integrated with the ensemble. Some methods such as CURL \cite{laskin2020curl} requires substantial changes to the base algorithm such as requiring momentum target network and data augmentation input, making their use in \textit{RENAULT} difficult.
	
	\subsection{Network Architecture}
	Before delving into each auxiliary task, we will describe our network architecture in detail. Our network consists of two main components: a feature/latent state extraction function $h(s) = z$ and a latent Q function $q(z)$. The feature function $h$ is a two layer convolution neural networks. Due to the use of dueling architecture, $q(z) = \frac{1}{|\sA|} \sum_{a=1}^{|\sA|} \text{adv}(z,a) + v(z)$, where $|\sA|$ is the action space of the problem, adv is a latent advantage function, and $v$ is a latent value function. Both $\text{adv}$ and $v$ are two layer fully-connected networks with ReLU as a first layer activation function. We use $\text{adv}_1$ ($v_1$) to denote the first layer of the $\text{adv}$ ($v$) function.
	
	\subsection{Model Learning as Auxiliary Tasks} \label{aux/model-learning}
	Our first auxiliary tasks are based on model learning. 
	Model learning is widely used in the context of model-based RL; but here we are using them as auxiliary tasks for DQN. They are easy to use with DQN; each task operates independently and requires no additional changes to the base algorithm. The detail on each model learning task is provided below.
	
	\noindent
	\textbf{Latent State Transition. } We learn a deterministic latent transition function which maps a latent state $z=h(s)$ and action $a$ to its next latent state through a parameterized function $T(z'|z,a;\theta)$. Given the actual next state $s'$, we seek to minimize the loss between the predicted latent state $z'$ and $h(s')$. We use smooth L1 loss \cite{huber1992robust} as our objective function.
	
	\noindent
	\textbf{Inverse Dynamic. } Inverse dynamic \cite{badia2019never} is a function that learns to predict the action that causes a transition from a certain state $s$ to another state $s'$. Given $z=h(s)$ and $z'=h(s')$, we seek to learn a parameterized function $T^{-1}(\hat{a}|z,z';\theta)$ by minimizing the loss of predicted action $\hat{a}$ with the real action $a$ via cross entropy loss.
	
	\noindent
	\textbf{Reward Function. } Let $w_1=\text{adv}_1(z,\cdot)$ and $w_2=v_1(z)$ be hidden representations corresponding to the output of the first layer of latent advantage function and latent value function for a latent state $z=h(s)$. Given an action $a$, and let $w=[w_1,w_2]$ be the concatenation of hidden representations $w_1$ and $w_2$, we seek to learn a reward function $r(w,a;\theta)$ by minimizing the distributional histogram loss \cite{imani2018improving} with the real reward $r(s,a)$. The use of reward function as an auxiliary prediction is not new \cite{jaderberg2016reinforcement}. However, we specify the task as a distributional prediction instead of classification, generalizing their formulation.
	
	\subsection{Object and Event based Auxiliary Tasks} \label{aux/objects-and-events}
	
	Our second set of auxiliary tasks aim to learn features that are useful for object and event based prediction.
	We propose two novel auxiliary tasks: change of moment and total change of intensity, to encourage learning features related to objects and events, respectively. The proposed tasks are simple, self-contained, and fairly general when objects and events are present, making them ideal for RENAULT. The detail of these two tasks are given as follows.
	
	\noindent
	\textbf{Change of moment. } We adopt the concept of moment in physics; a way to account for the distribution of physical quantities based on the product of distance and the quantities. In our case, we use pixels in place of the physical quantities. Thus, the moment corresponds to the distribution of the pixels, which roughly characterizes the distribution of the objects in the screen. For a given image state $s \in \R^{C\times W \times H}$ with channel $C$, width $W$, and height $H$, the moment is computed by $\mu(s) = \frac{1}{C} \sum_{c}^C \sum_{x,y}^{W,H} d(x,y) \times s_{c,x,y}$  where $d$ is a distance function to some reference point. We use coordinate $(0,0)$ as a reference point and euclidean distance as a distance function. We learn a function that captures the change of moment between a state $s$ and its corresponding next state $s'$ given an action $a$: $\delta_\mu(z,a;\theta) \approx \mu(s') - \mu(s)$, where $z=h(s)$ is the latent state of $s$. For stability, we normalize the change of moment by a squared total distance given by $d$. The function is optimized with smooth L1 loss.
	
	\noindent
	\textbf{Total change of intensity. } 
	An event is often characterized by the change of total pixel intensity. For example, objects disappearing due to destruction results in the loss of total pixel intensity, spawning of enemies increases the total pixel intensity, and an explosion triggers dramatic total change of intensity. As such, learning total change of intensity can be a sufficiently strong signal to learn to associate rewards with events. Similar idea regarding learning changes of pixels intensity has been explored by \cite{jaderberg2016reinforcement}, however, they propose to predict changes of intensity in the downsampled image patches using architecture similar to autoencoder. In contrast, we opt for a simpler objective of predicting the total change of intensity instead.
	
	Given an image state $s$ and its corresponding next state $s'$ given an action $a$, we denote the channel-mean of the state as $\hat{s}$ and the next state as $\hat{s}'$. In addition, we denote the latent state of a state $s$ as $z=h(s)$. We seek to learn a total change of intensity function $\delta_i(z,a;\theta) \approx || \hat{s} - \hat{s}' ||_2$. Since the value is bounded, we adopt the Histogram distributional loss similar to our reward function prediction.

	\section{Theoretical Analysis}
	
	In this section, we perform analysis to help understand the possible gains afforded by REN and RENAULT. We analyze the generalization error through bias-variance-covariance decomposition. Such an analysis is obviously inadequate for reinforcement learning, but we use it to potentially uncover good ways to use ensembles. We then run experiments to see which of the methods actually help for the case study.
	
	We seek to decompose the generalization error of ensembles into bias, variance, and covariance in the form similar to one proposed by \cite{ueda1996generalization}. Our analysis differ in that our decomposition focuses on ensembles learned through the use of a single dataset with a randomized learning algorithm. In contrast, their decomposition assumes that each member is trained on a different dataset. 
	
	We begin with the case of a single estimator. For the purpose of analysis, we assume our targets are generated by a fixed target function corresponding to the optimal Q-function $f^*(\vx)$, possibly corrupted by noise, and that our inputs are generated by a fixed policy. We want to learn a function $f(\vx;\theta)$ to approximate the unknown target function  by using a set of $N$ training samples $\{(\vx_i, y_i)\}_{i=1}^N$. For convenience, we denote $z^N=\{z_i\}_{i=1}^N$ to be a realization of a random set $Z^N=\{Z_i\}_{i=1}^N$, where $z_i=(x_i, y_i)$ and $Z_i=(X_i,Y_i)$. The parameter $\theta$ of $f(\vx; \theta)$ is learnt by a randomized algorithm $\gA(r, z^N)$, where $r$ is a random number drawn independently from a random set $\gR$. We will use $f(\vx; r, z^N)$ to refer to this parameterized function.
	
	Given a separate test vector $Z_0 = (X_0, Y_0)$, the generalization error of the function $f$ is
	$GE(f) = E_{Z^N, \gR}[ E_{Z_0}[(Y_0 - f(X_0;\gR, Z^N))^2] ]$.
	Let         
	\begin{align*}
	\Var(f|X_0) &= \E\Big[ \big( f(X_0;\gR, Z^N) - \E[ f(X_0;\gR, Z^N) ] \big)^2 \Big] \\
	\Bias(f|X_0) &= \E[ f(X_0;\gR, Z^N) ] - f^*(X_0)
	\end{align*}
	where $\E$ denotes the expectation $\E_{\gR, Z^N}$, and let $\sigma^2 = \E_{X_0,Y_0}[(f^*(X_0)-Y_0)^2]$ be an irreducible error.
	
	\begin{theorem}[Generalization error of random algorithm]
		The generalization error of the estimator $f$ can be decomposed as follows.
		\begin{align}
		GE(f) = E_{X_0}[ \Var(f|X_0) + \Bias(f|X_0)^2 ] + \sigma^2. \label{eq:generalization-error/bias-var}
		\end{align}
		\label{theorem:mse-decomposition}
	\end{theorem}
	
	All proofs are provided in Appendix \ref{appendix:proofs}.    
	
	Now, we will consider the case of ensemble estimators. Let there be $M$ estimators $\{f_m\}_{m=1}^M$; each estimator $f_m$ is trained using algorithm $\gA$ with an individual random number $r^{(m)}$. Note that $r^{(m)}$ is a realization of $\gR^{(m)}$. Given an input $\vx$, the output of the ensemble is:
	\begin{align}
	f_{ens}^{(M)}(\vx) = \frac{1}{M} \sum_{m=1}^{M} f(\vx;r^{(m)}, z^N)
	\end{align}
	Following \eqref{eq:generalization-error/bias-var}, the generalization error is given by
	\begin{align}
	GE(f_{ens}^{(M)}) = \E_{X_0}[ \Var(f_{ens}^{(M)}|X_0) + \Bias(f_{ens}^{(M)}|X_0)^2 ] + \sigma^2.
	\end{align}
	Let
	\begin{align*}
	\begin{split}
	\Cov(f_m, f_{m'}|X_0) = &\E\Big[\big( f(X_0;\gR^{(m)}, Z^N) - \E[ f(X_0;\gR^{(m)}, Z^N) ] \big) \\
	&\hspace{1em}\big( f(X_0;\gR^{(m')}, Z^N) - \E[ f(X_0;\gR^{(m')}, Z^N) ] \big) \Big].
	\end{split}
	\end{align*}
	\begin{align*}
	\begin{split}
	\overline{\Bias}(X_0) &= \frac{1}{M} \sum_{m=1}^M \Bias(f_m|X_0) \\
	\overline{\Var}(X_0) &= \frac{1}{M} \sum_{m=1}^M \Var(f_m|X_0) \\
	\overline{\Cov}(X_0) &= \frac{1}{M(M-1)} \sum_{m} \sum_{m'\ne m} \Cov(f_m, f_{m'}|X_0).
	\end{split} \label{eq:ensemble-bias-var-covar}
	\end{align*}
	
	\begin{theorem}[Generalization error of ensemble with random algorithm]
		The generalization error of the ensemble estimator $f_{ens}^{(M)}$ can be decomposed as:
		\begin{align}
		\begin{split}
		GE(f_{ens}^{(M)}) = \E_{X_0}\Big[ &\overline{\Bias}(X_0)^2 + \frac{1}{M} \overline{\Var}(X_0) + \big(1-\frac{1}{M}\big) \overline{\Cov} (X_0) \Big] + \sigma^2
		\end{split}
		\end{align}
		\label{theorem:ensemble-decomposition}
	\end{theorem} 
	
	Theorem \ref{theorem:mse-decomposition} and \ref{theorem:ensemble-decomposition} follow the proof from \cite{geman1992neural} and \cite{ueda1996generalization}, respectively. Although the results look similar, there are subtle differences in terms of the assumption with respect to the availability of multiple datasets and the use of randomness. 
	
	By analysing, the relationship between $\overline{\Var}(X_0)$ and $\overline{\Cov}(X_0)$, we obtain the following results about REN.
	\begin{theorem}
		$\overline{\Var}(X_0) \le \overline{\Cov}(X_0)$. Hence, if ensemble estimator $f_{ens}^{(M)}$ consists of $M$ identical estimators $f$ that differ only in the random numbers used, then $GE(f_{ens}^{(M)}) \le GE(f)$.
	\end{theorem}
	
	This result states that ensembles with the members trained the same way cannot hurt performance. The equal case can happen, e.g. when algorithm $\gA$ performs convex optimization, it will converge to the same minima regardless of random number used, resulting in all members of the ensemble being the same. In contrast, non-convex optimization algorithms such as SGD converges to a minima that depends on the randomness, thus will likely result in lower error due to reduction in the covariance term. Hence, REN will achieve \emph{at least} the same performance as Rainbow, and possibly better, under the idealized assumptions.
	
	Instead of training each member of the ensemble $f(\vx;r^{(m)}, z^N)$ separately, training the entire ensemble $f_{ens}^{(M)}(\vx)$ directly on the training set would result in lower training set error and possibly better generalization. We experiment with this as well in the case study.
	
	For a single network, auxiliary tasks usually reduce the variance as they provide additional information that help constrain the network. However, this may come at the cost of additional bias as the network needs to optimize for multiple objectives. To further understand the effects of auxiliary losses, we decompose the ensemble squared bias.
	\newcommand{\Cob}{\mathrm{Cob}}
	\begin{proposition}
		The $\overline{\Bias}(X_0)^2$ of ensemble estimator $f_{ens}^{(M)}$ can be decomposed as follows ($X_0$ is omitted for readability),
		\begin{align}
		\frac{1}{M^2} \Big[ \sum_{m}^M \Bias(f_m)^2 + \sum_{m} \sum_{m'\ne m} \Cob(f_m,f_m') \Big]
		\end{align}
		where $\Cob(f_m,f_m'|X_0)=\Bias(f_m|X_0)\Bias(f_{m'}|X_0)$ denote the product of bias; we refer to this as co-bias.
		\label{proposition:ensemble-bias-decomposition}
	\end{proposition} 
	
	This result suggests that lower ensemble generalization error can be obtained by increasing the number of negative co-bias by having a diverse set of positively and negatively biased members. This is more likely to be achieved in RENAULT if each member is assigned a unique set of auxiliary tasks. In contrast, assigning the same set of auxiliary tasks to each member results in $\overline{\Bias}(X_0) = \Bias(f|X_0)$ because $\forall_m \Bias(f_m|X_0) = \Bias(f|X_0)$. 
	
	\begin{proposition}
		Let $\bar{f}_{m,Z} = \E_{\gR^{(m)}}[ f_m | Z^N ]$ be a conditional expectation of $f_m$ over random number $\gR^{(m)}$ conditioned on $Z^N$. If the estimators are trained independently, then, $Cov(f_m,f_m') = E_{Z^N}[ (\bar{f}_{m,Z^N}-E[f_m])(\bar{f}_{m',Z^N}-E[f_{m'}]) ]$.\label{proposition:ensemble-covariance}
	\end{proposition}
	
	This result suggests that RENAULT may also reduce covariance if appropriate auxiliary tasks are assigned to each member. Otherwise, if all members are of the same model $f$, then $Cov(f_m,f_m') = E_{Z^N}[ (\bar{f}_{Z^N}-E[f])^2]$ which is the variance of the averaged estimator.
	
	\textbf{Limitations. }
	Our analysis assumes a fixed target; this is not available in RL. Instead, we have an estimate of the target (optimal Q value) based on TD return. The dataset in RL is also generated by a non-stationary policy, thus the distribution of the dataset keeps on changing during learning. Additionally, exploration also plays an important role in the learning of Q-function. Thus, it is important to note that our analysis will only provide partial insights regarding the methods; it serves to suggest possible ways to improve the algorithms, but the suggestions may not always help.
	\section{Experiments}
	
	In this section, we perform a set of experiments on REN and RENAULT. We compare them to prior methods. We examine whether joint training is better than independent training. We examine whether the auxiliary tasks help the ensembles and how to best use the auxiliary tasks.
	Before delving into each experiment, we will explain the problem domain of our case study, our architecture and hyperparameters, and our methods in detail.
	
	\textbf{Problem Domain. } We evaluate REN and RENAULT on a suite of Atari games from Atari Learning Environment (ALE) benchmark. We follow the evaluation procedure of \cite{kaiser2019model}; particularly, we limit the environment interaction to 100K interactions (400K frames with action repeated 4 frames) and evaluate on a subset of 26 games. We measure the raw performance score and human-normalized score, calculated as $100 \times (\text{Method score} - \text{Random score})/(\text{Human score} - \text{Random score})$.
	
	\textbf{Architecture and hyperparameters. } We follow the data-efficient Rainbow (DE-Rainbow) architecture and hyperparameters proposed by \cite{van2019use} and made no change to them. REN introduces a hyperparameter $M$ which controls the number of members of the ensemble. RENAULT introduces task and member specific hyperparameters $\alpha_{m,n}$ that control the strength of each auxiliary task. Additionally, each auxiliary task adopts different architecture; we give their description in Appendix \ref{appendix:auxiliary-task-architecture}.
	
	In our preliminary experiment, we found that $M < 5$ degrades performance and higher $M$ does not increase performance significantly while requiring more resources. Thus we fix $M=5$ throughout the experiment. 
	
	\textbf{Our Methods. } REN has two variants; one that is canonical according to our description in Section \ref{rainbow-ensemble}, and one that optimizes all members jointly, which we refer to as \emph{REN-J}. RENAULT also has two variants based on how we distribute the auxiliary tasks. The first variant, which we simply refer to as RENAULT, follows the suggestion of the preceding section to distribute the auxiliary tasks. As the number of auxiliary tasks equals to the number of members, we simply assign one \emph{unique} task for each member. In contrast, the second variant assigns \emph{all} tasks to each member, thus we call this variant \emph{RENAULT-all}. For simplicity, RENAULT uses $\alpha_{m,n}=1$ for all member $m$ and task $n$. 
	For RENAULT-all, we set $\alpha_{m,n} = \frac{1}{N_m}$, where $N_m=5$ is the number of auxiliary tasks for member $m$. This is to ensure that the auxiliary tasks do not overwhelm the main task.
	
	Further experimental details can be found in Appendix \ref{appendix:experimental-details}.

	\begin{table}\centering
		\caption{Performance on ATARI games on 100K interactions. Human Mean and Human Median indicate the mean and the median of the human-normalized score. The last two rows show the number of games won against DE-Rainbow and REN, respectively.}\label{tab:comparison-prior}
		\scriptsize
		\begin{tabular}{l|rrr|rr|rrr}\toprule
			&\textbf{SimPLe} &\textbf{OT-Rnbw} &\textbf{DE-Rnbw} &\textbf{REN} &\textbf{RNLT} &\textbf{REN-J} &\textbf{RNLT-all} \\\midrule
			alien &405.2 &824.7 &739.9 &828.7 &883.7 &800.3 &\textbf{890.0} \\
			amidar &88.0 &82.8 &188.6 &195.4 &\textbf{224.4} &120.2 &137.2 \\
			assault &369.3 &351.9 &431.2 &608.5 &\textbf{651.4} &504.0 &524.9 \\
			asterix &\textbf{1089.5} &628.5 &470.8 &578.3 &631.7 &645.0 &520.0 \\
			bank\_heist &8.2 &\textbf{182.1} &51.0 &63.3 &125.0 &64.7 &92.3 \\
			battle\_zone &5184.4 &4060.6 &10124.6 &\textbf{17500.0} &14233.3 &12666.7 &9000.0 \\
			boxing &9.1 &2.5 &0.2 &\textbf{10.9} &5.1 &5.2 &4.9 \\
			breakout &\textbf{12.7} &9.8 &1.9 &3.7 &3.4 &2.7 &3.0 \\
			chopper\_command &\textbf{1246.9} &1033.3 &861.8 &713.3 &896.7 &980.0 &563.3 \\
			crazy\_climber &\textbf{39827.8} &21327.8 &16185.3 &16523.3 &39460.0 &23613.3 &22123.3 \\
			demon\_attack &169.5 &711.8 &508.0 &759.3 &693.0 &665.5 &\textbf{822.7} \\
			freeway &20.3 &25.0 &27.9 &28.9 &29.3 &24.5 &\textbf{29.4} \\
			frostbite &254.7 &231.6 &866.8 &\textbf{2507.7} &1210.3 &2284.7 &1167.0 \\
			gopher &771.0 &\textbf{778.0} &349.5 &246.7 &542.7 &521.3 &323.3 \\
			hero &1295.1 &6458.8 &6857.0 &3817.2 &6568.8 &6499.3 &\textbf{7260.5} \\
			jamesbond &125.3 &112.3 &301.6 &518.3 &\textbf{628.3} &276.7 &420.0 \\
			kangaroo &323.1 &605.4 &779.3 &753.3 &540.0 &\textbf{893.3} &840.0 \\
			krull &\textbf{4539.9} &3277.9 &2851.5 &3105.1 &2831.3 &2667.2 &3827.0 \\
			kung\_fu\_master &\textbf{17257.2} &5722.2 &14346.1 &12576.7 &15703.3 &9616.7 &13423.3 \\
			ms\_pacman &762.8 &941.9 &1204.1 &1496.0 &\textbf{2002.7} &1240.7 &1705.0 \\
			pong &\textbf{5.2} &1.3 &-19.3 &-16.8 &-12.0 &-18.7 &-10.8 \\
			private\_eye &58.3 &\textbf{100.0} &97.8 &66.7 &66.7 &-35.2 &\textbf{100.0} \\
			qbert &559.8 &509.3 &1152.9 &1428.3 &583.3 &\textbf{2416.7} &1014.2 \\
			road\_runner &5169.4 &2696.7 &9600.0 &11446.7 &\textbf{13280.0} &5676.7 &7550.0 \\
			seaquest &370.9 &286.9 &354.1 &622.7 &\textbf{671.3} &555.3 &387.3 \\
			up\_n\_down &2152.6 &2847.6 &2877.4 &3568.0 &\textbf{4235.7} &3388.0 &3459.0 \\
			\midrule
			Human Mean &36.45\% &26.41\% &28.54\% &41.36\% &\textbf{45.64\%} &30.78\% &38.32\% \\
			Human Median &9.85\% &20.37\% &16.14\% &20.41\% &\textbf{25.08\%} &21.97\% &23.42\% \\
			\midrule
			vs DE-Rnbw &10 (-3) &12 (-1) &- &20 (+7) &\textbf{21 (+8)} &18 (+5) &19 (+6) \\
			vs REN &8 (-5) &10 (-3) &6 (-7) &- &\textbf{17 (+4)} &8 (-5) &13 (0) \\
			\bottomrule
		\end{tabular}
	\end{table}
	
	\begin{table}\centering
		\caption{Measurement of bias approximation, variance, covariance, irreducible error $\sigma^2$, and an approximation of generalization error ($\widehat{\text{GE}}$) of all methods. For Rainbow, $\widehat{\overline{\Bias}}$, $\overline{\Var}$, and $\overline{\Cov}$ denotes the estimator bias, variance, and covariance, respectively.}\label{tab:measurement}
		\scriptsize
		\begin{tabular}{l|rrrr|rr}\toprule
			&$\widehat{\overline{\Bias}}^2$ &$\overline{\Var}$ &$\overline{\Cov}$ &$\sigma^2$ &$\widehat{\text{GE}}$ \\\midrule
			REN &0.08 &1.09 &0.99 &1.02 &2.28 \\
			\midrule
			RENAULT &0.08 &0.82 &0.66 &0.63 &1.58 \\
			RENAULT-all &0.09 &0.81 &0.71 &0.74 &1.72 \\
			\midrule
			REN-J &0.07 &1.07 &0.51 &0.52 &1.41 \\
			Rainbow &0.08 &\textit{0.84} &- &0.70 &1.81 \\
			\bottomrule
		\end{tabular}
	\end{table}
	
	\subsection{Comparison to Prior Works}
	
	We compare the performance of REN and RENAULT to SimPLe \cite{kaiser2019model}, data-efficient Rainbow (DE-Rainbow) \cite{van2019use}, and Overtrained Rainbow (OT-Rainbow) \cite{kielak2020recent}. Two other recent works, CURL \cite{laskin2020curl} and SUNRISE \cite{lee2020sunrise} use game-dependent hyperparameters instead of using the same hyperparameters for all games, making their results not directly comparable to ours. The results are given in Table \ref{tab:comparison-prior}. We report the mean of three independent runs for our methods. We take the highest reported scores for SimPLe and human baselines, as they are reported differently in prior work \cite{van2016deep,kielak2020recent}.

	REN improves the performance of its baseline, data-efficient Rainbow on 20 out of 26 games and achieves better performance on 13 games. It also improves the mean and median human normalized performance $1.45 \times$ and $1.26 \times$, respectively. RENAULT further enhances the performance of REN, gaining $1.6 \times$ and $1.55 \times$ mean and median human normalized performance improvements. Additionally, it won on 21 games when compared to data-efficient Rainbow and exceeds REN's performance on 17 games. 
	
	\subsection{Bias-Variance-Covariance Measurements}
	
	To gain additional insights into our methods, we perform an empirical analysis by measuring their bias, variance, and covariance. Measuring bias requires the optimal Q-function which is unknown in RL. We measure the approximation to ensemble bias $\widehat{\Bias}(\theta)$ based on TD return in place of the real bias. We denote the ensemble bias based on this approximation as $\widehat{\overline{\Bias}}$. The detail of the measurements is given in Appendix \ref{appendix:bias-var-cov-measure}.
	
	The result of the measurements is given in Table \ref{tab:measurement}.
	
	We can see from Table~\ref{tab:measurement} that $\overline{\Cov}<\overline{\Var}$ in REN as expected from Proposition~\ref{theorem:ensemble-decomposition}. If the datasets used in the ensembles had been independent as well, we would have $\overline{\Cov}=0$, so the effects of independent randomization is more limited.
	RENAULT reduced the variance of REN as expected from the use of auxiliary tasks and running different tasks on different members of the ensemble appears to further reduce the covariance of RENAULT. 
	
	Comparison of REN and Rainbow also shows that our bias-variance-covariance measurements are not adequate for perfectly understanding the performance of the different algorithms. In particular, the generalization error of Rainbow is smaller than REN but REN had better performance. It is possible that the bias estimate using TD return is not a good proxy for the real bias; the TD return may be arbitrarily far from the optimal Q. Another possible reason could be that RL is much more than generalization error, which does not capture other aspects of RL such as exploration.
	
	\subsection{On Independent Training of Ensemble}
	
	Jointly optimizing all members of the ensemble would give better training error and possibly better generalization error. We compare the performance of REN with its variant, REN-J, that directly optimize the following loss: 
	\begin{align}
	\gL(\theta_{ens}) &= \E\big[ \KL[g_{s,a,r,s'}||Q_{ens}^{(M)}(s,a;\theta_{ens}) ] \big] 
	\end{align} where $\theta_{ens} = \{\theta_m\}_{m=1}^M$. 
	Since REN-J is essentially one single big neural network, it uses a single prioritized experience replay $\sD$ which is updated based on $\gL(\theta_{ens})$. 
	
	Table~\ref{tab:measurement} shows that REN-J indeed generalized better than REN. In particular, joint optimization substantially reduced $\overline{\Cov}$. However, REN surprisingly gives better overall performance compared REN-J. REN improves upon REN-J on 18 out of 26 games. It also improves the mean human normalized performance $1.34 \times$, although with a slight reduction of median performance of $0.93 \times$. When compared to data-efficient Rainbow, REN gains on two more games than REN-J.
	
	Contrary to expectation, in this case study, it is preferable to train an ensemble by optimizing each member independently, rather than treating the ensemble as a single monolithic neural network and optimize all members jointly to reduce its generalization error.
	
	\subsection{The Importance of Auxiliary Tasks}
	
	\begin{figure}[]
		\scriptsize
		\begin{center}
			\pgfplotstableread[row sep=\\,col sep=&]{
				interval & carT   \\
				NS      & 31.31 \\
				ID & 31.27  \\
				RF & 38.12  \\
				CI & 32.31  \\
				CM & 30.95  \\
			}\plotdata
			\begin{tikzpicture}
			\begin{axis}[
			width=.5\textwidth,
			height=.3\textwidth,
			ymax=45,
			ymin=20,
			yticklabel style={/pgf/number format/fixed},
			ylabel={Score (\%)},
			xtick=data,
			tickwidth=0,
			xticklabels from table={\plotdata}{interval},
			bar width=0.8,
			ybar=2pt,
			enlarge x limits={abs=0.55},
			nodes near coords,
			nodes near coords style={font=\tiny},
			]
			\addplot table [x expr=\coordindex,y=carT]{\plotdata};
			
			\coordinate (A) at (axis cs:0,28.54);
			\coordinate (O1) at (rel axis cs:0,0);
			\coordinate (O2) at (rel axis cs:1,0);
			\draw [black,sharp plot,dashed] (A -| O1) -- (A -| O2);
			
			\end{axis}
			\end{tikzpicture}
			\pgfplotstableread[row sep=\\,col sep=&]{
				interval & carT   \\
				NS      & 14 \\
				ID & 13  \\
				RF & 15  \\
				CI & 15  \\
				CM & 15  \\
			}\plotdata
			\begin{tikzpicture}
			\begin{axis}[
			width=.5\textwidth,
			height=.3\textwidth,
			ymax=16,
			ymin=12,
			yticklabel style={/pgf/number format/fixed},
			ylabel={\# Won},
			xtick=data,
			tickwidth=0,
			xticklabels from table={\plotdata}{interval},
			bar width=0.8,
			ybar=2pt,
			enlarge x limits={abs=0.55},
			nodes near coords,
			nodes near coords style={font=\tiny},
			]
			\addplot table [x expr=\coordindex,y=carT]{\plotdata};
			
			\coordinate (A) at (axis cs:0,13);
			\coordinate (O1) at (rel axis cs:0,0);
			\coordinate (O2) at (rel axis cs:1,0);
			\draw [black,sharp plot,dashed] (A -| O1) -- (A -| O2);
			
			\end{axis}
			\end{tikzpicture}
			
			\caption{Human normalized mean score (Left) and the number of games won (Right) of each member of the ensemble with (NS) latent next state prediction, (ID) inverse dynamic, (RF) reward function, (CI) total change of intensity, (CM) change of moment. As a reference, the performance of data-efficient Rainbow  is indicated by a dotted line.}
			\label{fig:aux-task-each}
			\vspace{-0.7cm}
		\end{center}
	\end{figure}
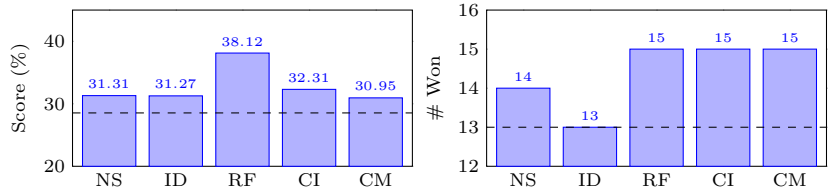
	
	Table \ref{tab:comparison-prior} shows that RENAULT improves REN on all counts. It wins on 17 games, gained $1.1\times$ and $1.23\times$ human mean and median normalized performance, as well as increasing the win count against Rainbow to 21 games. This demonstrates the significant benefit of augmenting ensembles with auxiliary tasks, at least in this case study. Moreover, this is achieved without any tuning to the auxiliary task hyperparameter $\alpha_{m,n}$; we simply set it to $1$ for all member $m$ and task $n$. We also simply distribute the auxiliary tasks as such that each member is augmented with one unique task. Careful tuning of the hyperparameter and task distribution may yield even better performance improvements.
	
	To understand the role of each auxiliary task, we analyze each of their contribution. Figure \ref{fig:aux-task-each} shows the contribution of each member of the ensemble that is endowed with a particular auxiliary task. It is interesting to see that although each task is weakly helpful (only offers modest performance improvement), they offer significant performance boost when combined with ensembles. The best performing auxiliary tasks in terms of games won are reward function, total change of intensity (CI), and change of moment (CM) prediction. This demonstrate the usefulness of our novel auxiliary tasks; we discuss this more in Appendix \ref{appendix:learning-reward-task}.
	
	In the opposite extreme, inverse dynamic (ID) seems to be less useful among the auxiliary tasks. Surprisingly, retraining RENAULT without ID reduces its performance substantially (see Appendix \ref{appendix:RENAULT-noID}). This suggests that ensemble improvements are not merely from individual gain, but also from diversity, through improved co-bias and covariance.
	
	\subsection{On Distributing the Auxiliary Tasks}
	
	Our theoretical result suggests that distributing the auxiliary tasks may be better than assigning all tasks on each member of the ensemble. To confirm this, we compare RENAULT with its variant which assigns \emph{all} auxiliary tasks to each member, RENAULT-all.
	
	Table \ref{tab:comparison-prior} shows that RENAULT-all performs worse than RENAULT, achieving lower mean and median human normalized score; this is in line with our expectation. While it may also be the case that suboptimal hyperparameters plays some roles in causing the performance degradation, this comparison is fair as we also did not perform tuning for RENAULT.
	
	Finally, RENAULT-all has larger ensemble bias and covariance compared to RENAULT in Table \ref{tab:measurement}. The larger ensemble bias could be because each network now has to optimize for more objectives. Propositions~\ref{proposition:ensemble-bias-decomposition} and \ref{proposition:ensemble-covariance} also suggest that RENAULT could be benefiting from reduced co-bias and covariance. The reduction could potentially be due to each member being less correlated when trained on the same dataset compared to RENAULT-all. 
	
	\section{Conclusions}
	
	In this work, we study ensembles and auxiliary tasks in the context of deep Q-learning. We proposed
	a simple agent that creates an ensemble of Q-functions based on Rainbow, and additionally augments it with auxiliary tasks. We provide theoretical analysis and an experimental case study. 
	Our methods improve significantly upon data-efficient Rainbow. We show that, although each auxiliary task only improves performance slightly, they significantly boost performance when combined using an ensemble.
	
	Our study focuses on the interaction between ensembles, auxiliary tasks, and DQN on learning. However, RL is a multi-faceted problem with many important components including exploration. Future work includes studying their interaction with exploration, which may provide important insights and answers to some of the questions which eludes our understanding in this work.
	
	\section*{Acknowledgements}
	We thank Lung Sin Kwee for useful discussions. This research is supported in part by the National Research Foundation, Singapore under its AI Singapore Program (AISG Award No: AISG2-RP-2020-016).

	%
	%
	%
	%
	\bibliographystyle{splncs04}
	\bibliography{references.bib}
	
	\newpage
	\appendix
	\setcounter{lemma}{0}
\setcounter{theorem}{0}
\setcounter{proposition}{0}

\section{Background} \label{appendix:background}

\subsection{Markov Decision Process}
A sequential decision problem is often modeled as a Markov Decision Process (MDP). An MDP is defined with a 5-tuple $<\sS,\sA,R,T,\gamma>$ where $\sS$ and $\sA$ denote the set of states and actions, $R$ and $T$ represent the reward and transition functions, and $\gamma \in [0,1)$ characterizes the discount factor of the MDP. Given an MDP, one can find an optimal policy $\pi^*$, such that following the policy, the return (sum of discounted rewards) $G = \sum_{t=1}^{\infty} \gamma^t R_t$ is maximized. When the reward function $R$ and the transition function $T$ are known, one can solve the MDP (i.e., find the optimal policy) by employing algorithms such as value iteration or planning. When they are not known, reinforcement learning (RL) can find the optimal policy by either learning the model and solving the MDP (model-based RL) or by directly learning the policy / value function (model-free RL).

\subsection{Deep Q Learning}
Deep Q Learning (DQN) \cite{mnih2015human} is one of the well known model-free RL algorithms. It is an instantiation of the Q learning algorithm, which directly learns a state-action value function of an MDP $Q^\pi(s,a)$ for a given state $s$ and action $a$. The goal of Q learning is to find the optimal Q-function $Q^*$ satisfying the Bellman optimality equation. 

DQN uses deep learning (neural networks) as a function approximation for $Q^\pi$. When parameterized by $\theta$ it is denoted as $Q(s,a;\theta)$. The Q function is used to select an action when interacting with the environment; of which the experience is accumulated in the replay buffer $\sD$. It learns to estimate the optimal Q-function by minimizing the following loss:
\begin{align}
\gL(\theta) = \E[(Q(s,a;\theta) - r + \gamma \max_{a'} Q(s',a';\theta'))^2]\label{eq/Q-loss-mse}
\end{align}
where $\E$ denotes the expectation over the realization of state, action, reward, and next state, i.e., $\E_{(s,a,r,s')\sim \sD}$ and $\theta'$ is the parameters of a target network (a past checkpoint of the network); target network is used to aid stability.

\section{Methods}
\subsection{The Value of Learning Objects and Events}

To illustrate the value of learning objects and events, we will have a look at two examples from our case study, the ATARI games. 

We take Freeway game as our first example. In Freeway, the objective is to help a chicken to cross a road that is packed with moving cars. The chicken has two actions: moving forward and backward. When a car hits the chicken, it will be thrown a few steps backward. Successfully crossing the road will gives a reward of one. In this case, learning objects implies learning the important factors corresponding to the assignment of rewards. Specifically, knowing the position of the chicken and each car allows the learner to efficiently learn the reward, and potentially the optimal policy.

The second example is Space Invaders games. In Space Invaders, the objective is to destroy all invading alien spaceships. We can move our spaceship in the horizontal axis and fire a laser canon; once fired, the laser will move slowly forward. If the alien spaceship is hit, then it will be destroyed; it will disappear from the screen. For each alien that we successfully hit, we will get a positive reward. In this case, knowing whether the alien is hit by the laser and thus is destroyed (disappeared from the screen) is important, as it defines the reward. This corresponds to knowing whether a certain event has occurred or not.

\section{Theoretical Analysis Proofs} \label{appendix:proofs}

Our bias-variance and bias-variance-covariance decompositions follow that of \cite{geman1992neural} and \cite{ueda1996generalization}.

\begin{lemma}
	Given random variable $X_m$, where $m = 1,...,M$, then the following inequality holds:
	\begin{align}
	\frac{1}{M(M-1)} \sum_{m' \ne m} \Cov(X_m, X_{m'}) \le \frac{1}{M} \sum_{m} \Var(X_m)
	\end{align}
	\label{lemma:cov-le-var}
\end{lemma}
\begin{proof}
	\begin{align}
	\Cov(X_m, X_{m'}) &\le \sqrt{\Var(X_m)\Var(X_{m'})} \label{eq:cauchy}\\
	&\le \frac{1}{2} \big( \Var(X_m) + \Var(X_{m'}) \big) \label{eq:GMAM}
	\end{align}
	Where \eqref{eq:cauchy} is due to Cauchy–Schwarz inequality and \eqref{eq:GMAM} is due to geometric mean $\le$ arithmetic mean. Thus,
	
	\begin{align}
	&\frac{1}{M(M-1)} \sum_{m' \ne m} \Cov(X_m, X_{m'}) \\
	&\le \frac{1}{M(M-1)} \sum_{m' \ne m} \frac{1}{2} \big( \Var(X_m) + \Var(X_{m'}) \big)\\
	&= \frac{1}{2M(M-1)} \sum_{m} \Var(X_m)\times 2(M-1) \\
	&= \frac{1}{M} \sum_{m} \Var(X_m) 
	\end{align}\qed
\end{proof}

Let,
\begin{align*}
\Var(f|X_0) &= \E\Big[ \big( f(X_0;\gR, Z^N) - \E[ f(X_0;\gR, Z^N) ] \big)^2 \Big] \\
\Bias(f|X_0) &= \E[ f(X_0;\gR, Z^N) ] - f^*(X_0)
\end{align*}
where $\E$ denotes the expectation $\E_{\gR, Z^N}$ and $\sigma^2 = \E_{X_0,Y_0}[(f^*(X_0)-Y_0)^2]$ is the irreducible error.

\begin{theorem}[Generalization error of random algorithm]
	The generalization error of the estimator $f$ can be decomposed as follows.
	\begin{align}
	GE(f) = E_{X_0}\big[ \Var(f|X_0) + \Bias(f|X_0)^2 \big] + \sigma^2.
	\end{align}
\end{theorem}
\begin{proof}
	Let,
	\begin{align}
	f^*(X_0) &= \E_{Y_0|X_0}[Y_0]\\
	\overline{f}(X_0) &= \E_{Z^N, \gR}[f(X_0;\gR, Z^N)] \label{eq:expected-estimator}\\
	f_{\gR, Z^N}(X_0) &= f(X_0;\gR, Z^N)
	\end{align}
	Then, the generalization error can be written as,
	\begin{align}
	&GE(f)\\ 
	&= E_{Z_0, Z^N, \gR}\Big[ \big( f_{\gR, Z^N}(X_0) - Y_0) \big)^2 \Big] \\
	&= E_{Z_0, Z^N, \gR}\Bigg[ \Big( \big( f_{\gR, Z^N}(X_0) - \overline{f}(X_0) \big) + \big( \overline{f}(X_0) - Y_0 \big) \Big)^2\Bigg]\\
	\begin{split}
	&= E_{Z_0, Z^N, \gR}\Big[ \big( f_{\gR, Z^N}(X_0) - \overline{f}(X_0) \big)^2 \Big] \\
	&\hspace{1em} + 2 E_{Z_0, Z^N, \gR} \Big[ \big( f_{\gR, Z^N}(X_0) - \overline{f}(X_0) \big) \big( \overline{f}(X_0) - Y_0 \big) \Big]  \\
	&\hspace{1em} + E_{Z_0}\Big[ \big( \overline{f}(X_0) - Y_0 \big)^2 \Big]
	\end{split}\\
	&= E_{Z_0, Z^N, \gR}\Big[ \big( f_{\gR, Z^N}(X_0) - \overline{f}(X_0) \big)^2 \Big] + E_{Z_0}\Big[ \big(\overline{f}(X_0) - Y_0 \big)^2 \Big]
	\label{eq:step-zero1}
	\end{align}
	where \eqref{eq:step-zero1} is due to:
	\begin{align}
	&E_{Z_0, Z^N, \gR}\Big[ \big( f_{\gR, Z^N}(X_0) - \overline{f}(X_0) \big) \big( \overline{f}(X_0) - Y_0 \big) \Big] \\
	&= E_{Z_0}\Big[ E_{Z^N, \gR} \big[ f_{\gR, Z^N}(X_0)  - \overline{f}(X_0) \big] \big(\overline{f}(X_0) - Y_0 \big) \Big]\\
	&= E_{Z_0}\Bigg[ \Big( E_{Z^N, \gR} \big[ f_{\gR, Z^N}(X_0) \big] - \overline{f}(X_0) \Big) \Big(\overline{f}(X_0) - Y_0 \Big) \Bigg]\\
	&= E_{Z_0}\Big[ \big( \overline{f}(X_0) - \overline{f}(X_0) \big) \big( \overline{f}(X_0) - Y_0 \big) \Big]\\
	&= 0
	\end{align}
	
	The second term in the \eqref{eq:step-zero1} can be broken down as follows:
	\begin{align}
	&E_{Z_0}\Big[ \big( \overline{f}(X_0) - Y_0 \big)^2 \Big]\\
	&= E_{Z_0}\Bigg[ \Big( \big( \overline{f}(X_0) - f^*(X_0) \big) + \big( f^*(X_0) - Y_0 ) \Big)^2 \Bigg]\\
	\begin{split}
	&= E_{Z_0}\Big[ \big( \overline{f}(X_0) - f^*(X_0) \big)^2 \Big] \\
	&\hspace{1em}+ E_{Z_0}\Big[ \big( f^*(X_0) - Y_0 \big)^2 \Big] \\
	&\hspace{1em}+ 2 E_{Z_0}\Big[ \big( \overline{f}(X_0) - f^*(X_0) \big)  \big( f^*(X_0) - Y_0 \big) \Big]
	\end{split}
	\end{align}
	where the third term equals to $0$, as shown below,
	\begin{align}
	&E_{Z_0}\Big[ \big( \overline{f}(X_0) - f^*(X_0) \big)  \big( f^*(X_0) - Y_0 \big) \Big]\\
	&= E_{X_0,Y_0}\Big[ \big( \overline{f}(X_0) - f^*(X_0) \big)  \big( f^*(X_0) - Y_0 \big) \Big]\\
	&= E_{X_0}\Bigg[ \E_{Y_0|X_0} \Big[ \big( \overline{f}(X_0) - f^*(X_0) \big)  \big( f^*(X_0) - Y_0 \big) \Big] \Bigg]\\
	&= E_{X_0}\Big[ \big( \overline{f}(X_0) - f^*(X_0) \big)  \E_{Y_0|X_0} \big[ f^*(X_0) - Y_0 \big] \Big]\\
	&= E_{X_0}\Big[ \big( \overline{f}(X_0) - f^*(X_0) \big)  \big( f^*(X_0) - \E_{Y_0|X_0} [ Y_0 ] \big) \Big]\\
	&= E_{X_0}\Big[ \big( \overline{f}(X_0) - f^*(X_0) \big)  \big( f^*(X_0) - f^*(X_0) \big) \Big]\\
	&= 0
	\end{align}
	This gives us the decomposed generalization error:
	\begin{align}
	&GE(f) \\
	\begin{split}
	&= E_{Z_0, Z^N, \gR}\Big[ \big( f_{\gR, Z^N}(X_0) - \overline{f}(X_0) \big)^2 \Big]\\
	&\hspace{1em}+ E_{Z_0}\Big[ \big( \overline{f}(X_0) - f^*(X_0) \big)^2 \Big] \\
	&\hspace{1em}+ E_{Z_0}\Big[ \big( f^*(X_0) - Y_0 \big)^2 \Big]
	\end{split} \\
	&= E_{X_0} \big[ \Var(f|X_0) + \Bias(f|X_0)^2 \big] + \sigma^2
	\end{align}
	\qed
\end{proof}

Let,
\begin{align*}
\begin{split}
\Cov&(f_m, f_{m'}|X_0) = \\
\E\Big[ &\big( f(X_0;\gR^{(m)}, Z^N) - \E[ f(X_0;\gR^{(m)}, Z^N) ] \big) \\
&\big( f(X_0;\gR^{(m')}, Z^N) - \E[ f(X_0;\gR^{(m')}, Z^N) ] \big) \Big].
\end{split}
\end{align*}
\begin{align*}
\begin{split}
\overline{\Bias}(X_0) &= \frac{1}{M} \sum_{m=1}^M \Bias(f_m|X_0) \\
\overline{\Var}(X_0) &= \frac{1}{M} \sum_{m=1}^M \Var(f_m|X_0) \\
\overline{\Cov}(X_0) &= \frac{1}{M(M-1)} \sum_{m} \sum_{m'\ne m} \Cov(f_m, f_{m'}|X_0).
\end{split}
\end{align*}

\begin{theorem}[Generalization error of ensemble with random algorithm]
	The generalization error of the ensemble estimator $f_{ens}^{(M)}$ can be decomposed as:
	\begin{align}
	\begin{split}
	GE(f_{ens}^{(M)}) = \E_{X_0}\Big[ &\overline{\Bias}(X_0)^2 + \frac{1}{M} \overline{\Var}(X_0) \\
	&+ \big(1-\frac{1}{M}\big) \overline{\Cov} (X_0) \Big] + \sigma^2
	\end{split}
	\end{align}
\end{theorem} 
\begin{proof}
	Let $\gR^{[1,M]}$ denotes $\gR^{(1)}, ..., \gR^{(M)}$, the variance of the ensemble can be decomposed as follows,
	\begin{align}
	&\Var(f_{ens}^{(M)}|X_0) \\
	&= \E_{Z^N, \gR^{[1,M]}}\Big[ \frac{1}{M} \sum_{m=1}^M f_m - \E_{Z^N, \gR^{[1,M]}} \big[ \frac{1}{M} \sum_{m=1}^M f_m \big]  \Big]\\
	\begin{split}
	&= \frac{1}{M^2} \sum_{m=1}^M \E_{Z^N, \gR^{(m)}}\Big[ \big( f_m - \E_{Z^N,  \gR^{(m)}}[f_m] \big)^2  \Big] \\
	&\hspace{1em}+ \frac{1}{M^2} \sum_{m} \sum_{m' \ne m} \E_{Z^N,  \gR^{(m)}, \gR^{(m')}}\Big[  
	\big( f_m - \E_{Z^N, \gR^{(m)}}[ f_m] \big) \\
	&\hspace{15em} \big( f_{m'} - \E_{Z^N, \gR^{(m')}}[ f_{m'}] \big) \Big] 
	\end{split}\\
	\begin{split}
	&= \Big( \frac{1}{M} \Big) \frac{1}{M} \sum_{m=1}^M \Var(f_m|X_0) \\
	&\hspace{1em} +\Big(1-\frac{1}{M}\Big) \frac{1}{M(M-1)} \sum_{m} \sum_{m'\ne m} \Cov(f_m, f_{m'}|X_0)
	\end{split}\\
	&= \frac{1}{M} \overline{\Var}(X_0) + \big(1-\frac{1}{M}\big) \overline{\Cov} (X_0)
	\end{align}
	As for the bias,
	\begin{align}
	\Bias(f_{ens}^{(M)}|X_0) &= \E_{Z^N, \gR^{[1,M]}}\Big[ \frac{1}{M} \sum_{m=1}^M f_m \Big] - f^*\\
	&= \frac{1}{M} \sum_{m=1}^M \E_{Z^N, \gR^{(m)}}\Big[ f_m - f^* \Big]\\
	&= \frac{1}{M} \sum_{m=1}^M \Bias(f_m|X_0) \\
	&= \overline{\Bias}(X_0)
	\end{align}
	Combining the two terms gives the $GE(f_{ens}^{(M)})$ as required.
	\qed
\end{proof}

\begin{theorem}
	$\overline{\Var}(X_0) \le \overline{\Cov}(X_0)$. Hence, if ensemble estimator $f_{ens}^{(M)}$ consists of $M$ identical estimators $f$ that differ only in the random numbers used, then $GE(f_{ens}^{(M)}) \le GE(f)$.
\end{theorem}
\begin{proof}
	$\overline{\Cov} (X_0)\leq \overline{\Var}(X_0)$ is due to Lemma \ref{lemma:cov-le-var}. Thus, the sum of the second and the third term of $GE(f_{ens}^{(M)})$ is bounded by $\overline{\Var}(X_0)$. As $f_{ens}^{(M)}$ consists of $M$ identical estimators $f$, $\overline{\Var}(X_0) = \frac{1}{M} \sum_{m=1}^M \Var(f_m|X_0) = \Var(f|X_0)$, thus $GE(f_{ens}^{(M)}) \le GE(f)$ as required.
	\qed
\end{proof}

\begin{proposition}
	The $\overline{\Bias}(X_0)^2$ of ensemble estimator $f_{ens}^{(M)}$ can be decomposed as follows ($X_0$ is omitted for readability),
	\begin{align}
	\frac{1}{M^2} \Big[ \sum_{m}^M \Bias(f_m)^2 + \sum_{m} \sum_{m'\ne m} \Cob(f_m,f_m') \Big]
	\end{align}
	where $\Cob(f_m,f_m'|X_0)=\Bias(f_m|X_0)\Bias(f_{m'}|X_0)$ denote the product of bias; we refer to this as co-bias.
\end{proposition} 
\begin{proof}
	\begin{align}
	&\overline{\Bias}(X_0)^2\\
	&= \big( \frac{1}{M} \sum_{m}^M \Bias(f_m|X_0) \big)^2\\
	&= \frac{1}{M^2} \Big[ \sum_{m}^M \Bias(f_m|X_0)^2 + \sum_{m} \sum_{m\ne m'} \Bias(f_m|X_0)\Bias(f_{m'}|X_0) \Big]\\
	&= \frac{1}{M^2} \Big[ \sum_{m} \Bias(f_m|X_0)^2 + \sum_{m} \sum_{m'\ne m} \Cob(f_m,f_m'|X_0) \Big]
	\end{align}
	\qed
\end{proof}

\begin{proposition}
	Let $\bar{f}_{m,Z} = \E_{\gR^{(m)}}[ f_m | Z^N ]$ be a conditional expectation of $f_m$ over random number $\gR^{(m)}$ conditioned on $Z^N$. If the estimators are trained independently, then, $Cov(f_m,f_m')$ can be expressed as $E_{Z^N}[ (\bar{f}_{m,Z^N}-E[f_m])(\bar{f}_{m',Z^N}-E[f_{m'}]) ]$.
\end{proposition}
\begin{proof}
	Let $\E[ f_{m} ] = \E_{\gR^{(m)}, Z^N}[ f(X_0;\gR^{(m)}, Z^N) ]$,
	\begin{align}
	\Cov(f_m, &f_{m'}|X_0)  \\
	\begin{split}
	=\E_{Z^N}\Bigg[ &\E_{R^{(m)}, R^{(m')}} \Big[ \big( f(X_0;\gR^{(m)}, Z^N) - \E[ f_m ] \big) \\
	&\big( f(X_0;\gR^{(m')}, Z^N) - \E[ f_{m'} ] \big) \Big] \Bigg]
	\end{split} \label{eq:cov-independent1}\\
	\begin{split}
	=\E_{Z^N}\Bigg[ 
	&\E_{R^{(m)}} \Big[ \big( f(X_0;\gR^{(m)}, Z^N) - \E[ f_m ] \big) \Big]\\
	&\E_{R^{(m')}}\Big[ \big( f(X_0;\gR^{(m')}, Z^N) - \E[ f_{m'} ] \big) \Big] \Bigg]
	\end{split} \label{eq:cov-independent2}\\
	\begin{split}
	=\E_{Z^N}\Bigg[ 
	&\Big( \E_{R^{(m)}} \big[ f(X_0;\gR^{(m)}, Z^N) \big] - \E[ f_m ] \Big) \\
	&\Big( \E_{R^{(m')}} \big[ f(X_0;\gR^{(m')}, Z^N) \big] - \E[ f_{m'} ] \Big) \Bigg]
	\end{split}\\
	\begin{split}
	=\E_{Z^N}\Big[ 
	&\big( \E_{\gR^{(m)}}[ f_m | Z^N ] - \E[ f_m ] \big) \\
	&\big( \E_{\gR^{(m')}}[ f_{m'} | Z^N ] - \E[ f_{m'} ] \big) \Big]
	\end{split}\\
	\begin{split}
	=\E_{Z^N}\Big[ \big(
	&\bar{f}_{m,Z} - \E[ f_m ] \big) \big( \bar{f}_{m',Z} - \E[ f_{m'} ] \big) \Big]
	\end{split}
	\end{align}
	where \eqref{eq:cov-independent2} is due to independence of random variable.
	\qed
\end{proof}

\section{Bias-Variance-Covariance Measurements} \label{appendix:bias-var-cov-measure}
Measuring bias requires the optimal Q-function which is unknown in RL. Thus, we opt to measure the following value as a proxy,
\begin{align}
\widehat{\Bias}(\theta) = \E\Big[Q(s,a;\theta) - G(s,a,r,s') \Big]
\end{align}
where $G(s,a,r,s') = \E_{\theta'}[ r + \gamma *  \max_{a'} Q(s',a'; \theta') ]$. 

Since we are dealing with ensembles, we can measure the ensemble bias, variance, and covariance following,
\begin{align*}
\begin{split}
\overline{\Bias}(X_0) &= \frac{1}{M} \sum_{m=1}^M \Bias(f_m|X_0) \\
\overline{\Var}(X_0) &= \frac{1}{M} \sum_{m=1}^M \Var(f_m|X_0) \\
\overline{\Cov}(X_0) &= \frac{1}{M(M-1)} \sum_{m} \sum_{m'\ne m} \Cov(f_m, f_{m'}|X_0).
\end{split} \label{eq:ensemble-bias-var-covar}
\end{align*}

We substitute the measurement for the ensemble bias with $\widehat{\overline{\Bias}} = \frac{1}{M}\sum_m \widehat{\Bias}(\theta_m)$.

To obtain each value, we first gather 1000 unseen transitions, i.e., transitions not in the training buffer (off-buffer transitions), following the epsilon-greedy policy of the measured Q-function; this is done for each game. For a given game, we compute the measurement for each individual transition. If the expected estimator (\eqref{eq:expected-estimator}) is required, we compute it by averaging the output of all runs for that particular transition. Finally, for each value, we take the average across the 1000 transitions, followed by averaging across 26 games. 

\section{Additional Experiment Details} \label{appendix:experimental-details}

\subsection{Auxiliary Tasks Architecture} \label{appendix:auxiliary-task-architecture}
Each auxiliary task adopts different architecture described as follows:
\begin{itemize}
	\item \textbf{Latent State Transition} consists of a single 64 channel 3x3 same-padding convolution layer with skip connection and ReLU activation followed by one resnet block of two 64 channel 3x3 same-padding  convolution layers. The action is encoded as a channel which is concatenated with the state.
	\item \textbf{Inverse Dynamic} (\textbf{Change of moment}, \textbf{Total change of intensity}) consists of 2 layer neural networks with 576 and 256 neurons, respectively. The hidden layer uses ReLU activation function.
	\item \textbf{Reward Function} consists of a single layer networks with 512 input neurons. We use 3 atoms and $\sigma=0.1$ for our distributional loss (chosen arbitrarily).
	\item \textbf{Total change of intensity}'s distributional loss uses 84 atoms and maintain a ratio of $0.5$ following \cite{imani2018improving}. 
\end{itemize}

\subsection{Prioritized Experience Replay}
Although we maintain $m$ prioritized replay buffer, the transitions are shared across buffers, allowing for low memory footprint. 

\subsection{ATARI environment}
The ATARI environment is available from the OpenAI ATARI-Py\footnote{https://github.com/openai/atari-py}.

\textbf{Preprocessing. } We follow the preprocessing step of \cite{van2019use}. First, we convert the RGB image into a greyscale image and downsample it to $84 \times 84$ pixels. We repeat the same action for 4 frames, and take the max pooling of the last two frames. We define our state as a concatenation of four of such max-pooled frame over the last four environment steps. We also clipped our reward within the range of $[-1,1]$.

\subsection{Hyperparameters and Random Seeds}
All hyperparamaters are given in Table \ref{tab:hyperparameters}. In our preliminary experiment, we found that $M < 5$ degrades performance and higher $M$ does not increase performance significantly while requiring more resources. Thus we fix $M=5$ throughout the experiment. Apart from that, we did not perform any hyperparameter search.

All experiments are run three times with the following random seeds: $\{123, 234, 345\}$. 

\subsection{Algorithm Runtime}

The runtime of the algorithm for one game is given in Table \ref{tab:runtime}. $T$ is around one hour when using the machines in \ref{computing-infrastructure}.

\begin{table}[htp]\centering
	\caption{Algorithm Runtime.}\label{tab:runtime}
	\scriptsize
	\begin{tabular}{lr}\toprule
		\textbf{Name} &\textbf{Runtime ($T$)} \\\midrule
		Rainbow & $T$ \\
		REN-J & $4T$ \\
		REN & $5T$ \\
		RENAULT & $7T$ \\
		RENAULT-all & $10T$ \\
		\bottomrule
	\end{tabular}
\end{table}

\begin{table}[!htp]\centering
	\caption{Hyperparameters of Rainbow, REN, REN-J, RENAULT, and RENAULT-all. NS: latent next state, ID: inverse dynamic, RF: reward function, CI: total change of intensity, and CM: change of moment. The Rainbow hyperparameters is adopted from \cite{van2019use}}\label{tab:hyperparameters}
	\scriptsize
	\begin{xltabular}{\textwidth}{lX}\toprule
		\textbf{Hyperparameter} &\textbf{Setting (All)} \\\midrule
		Grey-scaling &TRUE \\
		Observation down-sampling &(84, 84) \\
		Frames stacked &4 \\
		Action repetitions &4 \\
		Reward clipping &[-1, 1] \\
		Terminal on loss of life &TRUE \\
		Max frames per episode &108K \\
		Update &Distributional Double Q \\
		Target network update period &every 2000 updates \\
		Support of Q-distribution &51 bins \\
		Discount factor &0.99 \\
		Minibatch size &32 \\
		Optimizer &Adam \\
		Optimizer: first moment decay &0.9 \\
		Optimizer: second moment decay &0.999 \\
		Optimizer: e &0.00015 \\
		Max gradient norm &10 \\
		Priority exponent &0.5 \\
		Priority correction &0.4 $\rightarrow$ 1 \\
		Noisy nets parameter &0.1 \\
		Training frames &400,000 \\
		Min replay size for sampling &1600 \\
		Memory size &unbounded \\
		Replay period every &1 steps \\
		Multi-step return length &20 \\
		Q network: channels &32, 64 \\
		Q network: filter size &5 × 5, 5 × 5 \\
		Q network: stride &5, 5 \\
		Q network: hidden units &256 \\
		Optimizer: learning rate &0.0001 \\
		\midrule
		\textbf{Hyperparameter} &\textbf{(REN, RENAULT, RENAULT-all)} \\
		Size of ensemble (\# members) &5 \\
		Number of priority replay &5 \\
		Independent training of members &TRUE \\
		\midrule
		\textbf{Hyperparameter} &\textbf{(REN-J)} \\
		Size of ensemble (\# members) &5 \\
		Number of priority replay &1 \\
		Independent training of members &FALSE \\
		\midrule
		\textbf{Hyperparameter} &\textbf{(RENAULT)} \\
		Strength parameter ($\alpha_{m,n}$ for all member $m$ and task $n$) &1 \\
		Auxiliary tasks &(NS), (ID), (RF), (CI), (CM) \\
		\midrule
		\textbf{Hyperparameter} &\textbf{(RENAULT-all)} \\
		Strength parameter ($\alpha_{m,n}$ for all member $m$ and task $n$) &1/5 \\
		Auxiliary tasks &\seqsplit{(NS, ID, RF, CI, CM), (NS, ID, RF, CI, CM), (NS, ID, RF, CI, CM), (NS, ID, RF, CI, CM), (NS, ID, RF, CI, CM)} \\
		\bottomrule
	\end{xltabular}
\end{table}

\subsection{Computing Infrastructure} \label{computing-infrastructure}

We use the computing infrastructure described in Table \ref{tab:computing-infra}.

\subsection{Source Code}

Our code is open source and available at \url{https://github.com/NUS-LID/RENAULT}. The instruction on how to use the code is given in the \emph{README.md}. Our implementation is based on \url{https://github.com/Kaixhin/Rainbow}. 

\begin{table}[!htp]\centering
	\caption{Computing Infrastructure}\label{tab:computing-infra}
	\scriptsize
	\begin{tabular}{lrrrr}\toprule
		\textbf{ID} &\textbf{CPU} &\textbf{GPU} &\textbf{Num Machines}  \\\midrule
		M1 &Intel(R) Xeon(R) Gold 6240 CPU @ 2.60GHz &NVIDIA GeForce RTX 2080 Ti &2 \\
		M2 &Intel(R) Xeon(R) Silver 4116 CPU @ 2.10GHz &NVIDIA TITAN RTX &5 \\
		M3 &Intel(R) Xeon(R) Silver 4116 CPU @ 2.10GHz &NVIDIA Tesla T4 &5 \\
		\bottomrule
	\end{tabular}
\end{table}

\section{Additional Results}

\subsection{On Learning Reward Related Tasks} \label{appendix:learning-reward-task}

It is interesting to point out that one can view CI and CM as tasks to help learn feature useful for predicting reward; change in object positions and triggering of events are often associated with reward. Along with the significant gain afforded by the reward function prediction task, these potentially suggest that task related to reward learning is generally beneficial for efficiently learning Q-functions. 

\subsection{RENAULT without Inverse Dynamic} \label{appendix:RENAULT-noID}

Inverse dynamic (ID) seems to be less useful among all the auxiliary tasks. To understand whether ID plays a role in the ensemble, we remove it from RENAULT and retrain the ensemble from scratch. We denote this variant as \emph{RENAULT-noID}. Figure \ref{fig:mean-median-RENAULT-vs-RENAULT-noID} shows the performance different between RENAULT and RENAULT-noID. We can see that removing the ID task is detrimental to the performance of RENAULT. This might suggests that the auxiliary tasks are all useful, and the gain afforded by them is possibly due to increased diversity; not only from individual gain of performance.

\subsection{Detailed Performance}

The detailed performance of each method is given in Table \ref{tab:detail-REN}, \ref{tab:detail-RENAULT}, \ref{tab:detail-REN-J}, \ref{tab:detail-RENAULT-all}, and \ref{tab:detail-RENAULT-noID}. Additionally, the detail of the performance for each auxiliary task is given in Table \ref{tab:detail-Aux}.

\begin{figure}[]
	\scriptsize
	\begin{center}
		\pgfplotstableread[row sep=\\,col sep=&]{
			interval & RENAULT  & RENAULT (-ID)   \\
			Mean      & 45.64 & 38.57 \\
			Median & 25.08 & 21.62 \\
		}\plotdata
		\begin{tikzpicture}
		\begin{axis}[
		width=.8\textwidth,
		height=.5\textwidth,
		ymax=60,
		ymin=20,
		yticklabel style={/pgf/number format/fixed},
		ylabel={Score (\%)},
		xtick=data,
		tickwidth=0,
		xticklabels from table={\plotdata}{interval},
		tickwidth=0,
		bar width=0.3,
		ybar=2pt,
		enlarge x limits={abs=0.55},
		nodes near coords,
		nodes near coords style={font=\tiny},
		legend pos=north east,
		legend style={legend columns=-1},
		]
		\addplot table [x expr=\coordindex,y=RENAULT]{\plotdata};
		\addplot table [x expr=\coordindex,y=RENAULT (-ID)]{\plotdata};
		\legend{RENAULT, RENAULT-noID}
		\end{axis}
		\end{tikzpicture}
		
		\caption{Human normalized mean and median score of RENAULT with and without inverse dynamic (RENAULT-noID).}
		\label{fig:mean-median-RENAULT-vs-RENAULT-noID}
	\end{center}
\end{figure}
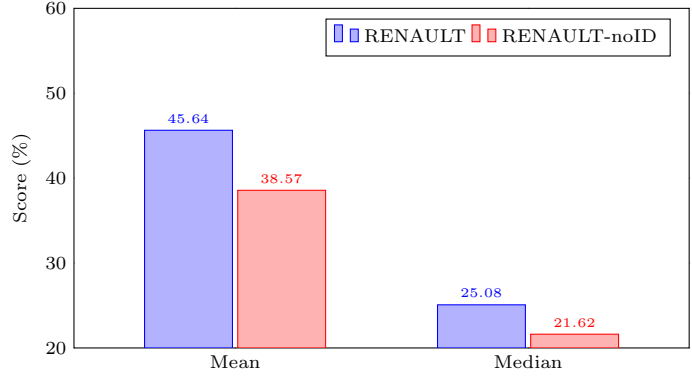

\begin{table}[!htp]\centering
	\caption{REN}\label{tab:detail-REN}
	\scriptsize
	\begin{tabular}{lrrrrrr}\toprule
		&\multicolumn{3}{c}{\textbf{Runs}} &\multirow{2}{*}{\textbf{Mean}} &\multirow{2}{*}{\textbf{Std}} \\\cmidrule{2-4}
		&\textbf{1} &\textbf{2} &\textbf{3} & & \\\midrule
		alien &822.0 &949.0 &715.0 &828.7 &117.1 \\
		amidar &221.0 &179.3 &185.8 &195.4 &22.4 \\
		assault &574.0 &618.4 &633.0 &608.5 &30.7 \\
		asterix &330.0 &505.0 &900.0 &578.3 &292.0 \\
		bank\_heist &56.0 &61.0 &73.0 &63.3 &8.7 \\
		battle\_zone &16900.0 &3300.0 &32300.0 &17500.0 &14509.3 \\
		boxing &22.0 &2.6 &8.1 &10.9 &10.0 \\
		breakout &2.4 &6.1 &2.6 &3.7 &2.1 \\
		chopper\_command &1070.0 &550.0 &520.0 &713.3 &309.2 \\
		crazy\_climber &19300.0 &4570.0 &25700.0 &16523.3 &10835.2 \\
		demon\_attack &742.5 &531.5 &1004.0 &759.3 &236.7 \\
		freeway &30.3 &29.2 &27.2 &28.9 &1.6 \\
		frostbite &2777.0 &3300.0 &1446.0 &2507.7 &955.9 \\
		gopher &186.0 &276.0 &278.0 &246.7 &52.5 \\
		hero &4277.5 &2995.0 &4179.0 &3817.2 &713.7 \\
		jamesbond &245.0 &1030.0 &280.0 &518.3 &443.5 \\
		kangaroo &1120.0 &600.0 &540.0 &753.3 &319.0 \\
		krull &3329.1 &3297.3 &2688.9 &3105.1 &360.8 \\
		kung\_fu\_master &7220.0 &23200.0 &7310.0 &12576.7 &9200.2 \\
		ms\_pacman &2286.0 &961.0 &1241.0 &1496.0 &698.3 \\
		pong &-14.8 &-18.2 &-17.4 &-16.8 &1.8 \\
		private\_eye &0.0 &100.0 &100.0 &66.7 &57.7 \\
		qbert &540.0 &3040.0 &705.0 &1428.3 &1398.2 \\
		road\_runner &12300.0 &11010.0 &11030.0 &11446.7 &739.1 \\
		seaquest &564.0 &686.0 &618.0 &622.7 &61.1 \\
		up\_n\_down &2968.0 &3650.0 &4086.0 &3568.0 &563.5 \\
		\bottomrule
	\end{tabular}
\end{table}

\begin{table}[!htp]\centering
	\caption{RENAULT}\label{tab:detail-RENAULT}
	\scriptsize
	\begin{tabular}{lrrrrrr}\toprule
		&\multicolumn{3}{c}{\textbf{Runs}} &\multirow{2}{*}{\textbf{Mean}} &\multirow{2}{*}{\textbf{Std}} \\\cmidrule{2-4}
		&\textbf{1} &\textbf{2} &\textbf{3} & & \\\midrule
		alien &934.0 &688.0 &1029.0 &883.7 &176.0 \\
		amidar &245.7 &179.7 &247.8 &224.4 &38.7 \\
		assault &648.9 &663.6 &641.6 &651.4 &11.2 \\
		asterix &605.0 &510.0 &780.0 &631.7 &137.0 \\
		bank\_heist &155.0 &90.0 &130.0 &125.0 &32.8 \\
		battle\_zone &19100.0 &9400.0 &14200.0 &14233.3 &4850.1 \\
		boxing &5.5 &6.9 &3.0 &5.1 &2.0 \\
		breakout &0.4 &7.0 &2.9 &3.4 &3.3 \\
		chopper\_command &1050.0 &970.0 &670.0 &896.7 &200.3 \\
		crazy\_climber &44820.0 &41370.0 &32190.0 &39460.0 &6528.0 \\
		demon\_attack &695.5 &586.0 &797.5 &693.0 &105.8 \\
		freeway &30.0 &29.5 &28.5 &29.3 &0.8 \\
		frostbite &1657.0 &255.0 &1719.0 &1210.3 &827.9 \\
		gopher &712.0 &396.0 &520.0 &542.7 &159.2 \\
		hero &7370.0 &4832.5 &7504.0 &6568.8 &1505.2 \\
		jamesbond &445.0 &1120.0 &320.0 &628.3 &430.4 \\
		kangaroo &620.0 &600.0 &400.0 &540.0 &121.7 \\
		krull &2941.2 &2334.7 &3217.9 &2831.3 &451.7 \\
		kung\_fu\_master &7880.0 &23640.0 &15590.0 &15703.3 &7880.6 \\
		ms\_pacman &1448.0 &3252.0 &1308.0 &2002.7 &1084.2 \\
		pong &2.4 &-17.7 &-20.8 &-12.0 &12.6 \\
		private\_eye &0.0 &100.0 &100.0 &66.7 &57.7 \\
		qbert &545.0 &440.0 &765.0 &583.3 &165.9 \\
		road\_runner &11520.0 &16150.0 &12170.0 &13280.0 &2506.7 \\
		seaquest &740.0 &718.0 &556.0 &671.3 &100.5 \\
		up\_n\_down &5389.0 &2957.0 &4361.0 &4235.7 &1220.8 \\
		\bottomrule
	\end{tabular}
\end{table}

\begin{table}[!htp]\centering
	\caption{REN-J}\label{tab:detail-REN-J}
	\scriptsize
	\begin{tabular}{lrrrrrr}\toprule
		&\multicolumn{3}{c}{\textbf{Runs}} &\multirow{2}{*}{\textbf{Mean}} &\multirow{2}{*}{\textbf{Std}} \\\cmidrule{2-4}
		&\textbf{1} &\textbf{2} &\textbf{3} & & \\\midrule
		alien &954.0 &864.0 &583.0 &800.3 &193.5 \\
		amidar &96.7 &127.6 &136.2 &120.2 &20.8 \\
		assault &413.7 &670.0 &428.4 &504.0 &143.9 \\
		asterix &685.0 &710.0 &540.0 &645.0 &91.8 \\
		bank\_heist &75.0 &89.0 &30.0 &64.7 &30.8 \\
		battle\_zone &5200.0 &10600.0 &22200.0 &12666.7 &8686.4 \\
		boxing &8.4 &4.8 &2.3 &5.2 &3.1 \\
		breakout &3.1 &2.3 &2.8 &2.7 &0.4 \\
		chopper\_command &1080.0 &660.0 &1200.0 &980.0 &283.5 \\
		crazy\_climber &24620.0 &18930.0 &27290.0 &23613.3 &4269.9 \\
		demon\_attack &534.0 &555.0 &907.5 &665.5 &209.8 \\
		freeway &22.8 &25.4 &25.3 &24.5 &1.5 \\
		frostbite &1880.0 &2511.0 &2463.0 &2284.7 &351.3 \\
		gopher &780.0 &480.0 &304.0 &521.3 &240.7 \\
		hero &7384.5 &7223.0 &4890.5 &6499.3 &1395.6 \\
		jamesbond &295.0 &325.0 &210.0 &276.7 &59.7 \\
		kangaroo &660.0 &1260.0 &760.0 &893.3 &321.5 \\
		krull &2467.7 &2651.1 &2882.7 &2667.2 &208.0 \\
		kung\_fu\_master &18980.0 &6330.0 &3540.0 &9616.7 &8228.0 \\
		ms\_pacman &479.0 &1108.0 &2135.0 &1240.7 &835.9 \\
		pong &-20.3 &-15.4 &-20.3 &-18.7 &2.8 \\
		private\_eye &0.0 &-205.6 &100.0 &-35.2 &155.8 \\
		qbert &5635.0 &905.0 &710.0 &2416.7 &2788.9 \\
		road\_runner &9940.0 &4780.0 &2310.0 &5676.7 &3893.2 \\
		seaquest &344.0 &726.0 &596.0 &555.3 &194.2 \\
		up\_n\_down &3755.0 &3087.0 &3322.0 &3388.0 &338.9 \\
		\bottomrule
	\end{tabular}
\end{table}

\begin{table}[!htp]\centering
	\caption{RENAULT-all}\label{tab:detail-RENAULT-all}
	\scriptsize
	\begin{tabular}{lrrrrrr}\toprule
		&\multicolumn{3}{c}{\textbf{Runs}} &\multirow{2}{*}{\textbf{Mean}} &\multirow{2}{*}{\textbf{Std}} \\\cmidrule{2-4}
		&\textbf{1} &\textbf{2} &\textbf{3} & & \\\midrule
		alien &1193.0 &704.0 &773.0 &890.0 &264.7 \\
		amidar &121.4 &138.8 &151.4 &137.2 &15.1 \\
		assault &473.4 &589.0 &512.4 &524.9 &58.8 \\
		asterix &460.0 &670.0 &430.0 &520.0 &130.8 \\
		bank\_heist &90.0 &147.0 &40.0 &92.3 &53.5 \\
		battle\_zone &13500.0 &9300.0 &4200.0 &9000.0 &4657.3 \\
		boxing &4.4 &0.8 &9.6 &4.9 &4.4 \\
		breakout &1.8 &4.4 &2.7 &3.0 &1.3 \\
		chopper\_command &820.0 &270.0 &600.0 &563.3 &276.8 \\
		crazy\_climber &23610.0 &13220.0 &29540.0 &22123.3 &8260.9 \\
		demon\_attack &594.0 &732.0 &1142.0 &822.7 &285.0 \\
		freeway &29.7 &30.0 &28.4 &29.4 &0.9 \\
		frostbite &2250.0 &634.0 &617.0 &1167.0 &937.9 \\
		gopher &284.0 &398.0 &288.0 &323.3 &64.7 \\
		hero &6856.5 &7560.0 &7365.0 &7260.5 &363.2 \\
		jamesbond &505.0 &340.0 &415.0 &420.0 &82.6 \\
		kangaroo &700.0 &1220.0 &600.0 &840.0 &332.9 \\
		krull &5018.4 &3045.0 &3417.6 &3827.0 &1048.5 \\
		kung\_fu\_master &7190.0 &15480.0 &17600.0 &13423.3 &5501.3 \\
		ms\_pacman &1857.0 &1217.0 &2041.0 &1705.0 &432.5 \\
		pong &-12.3 &-7.3 &-12.9 &-10.8 &3.1 \\
		private\_eye &100.0 &100.0 &100.0 &100.0 &0.0 \\
		qbert &787.5 &492.5 &1762.5 &1014.2 &664.6 \\
		road\_runner &6160.0 &6470.0 &10020.0 &7550.0 &2144.7 \\
		seaquest &398.0 &412.0 &352.0 &387.3 &31.4 \\
		up\_n\_down &3623.0 &4113.0 &2641.0 &3459.0 &749.6 \\
		\bottomrule
	\end{tabular}
\end{table}

\begin{table}[!htp]\centering
	\caption{RENAULT-noID}\label{tab:detail-RENAULT-noID}
	\scriptsize
	\begin{tabular}{lrrrrrr}\toprule
		&\multicolumn{3}{c}{\textbf{Runs}} &\multirow{2}{*}{\textbf{Mean}} &\multirow{2}{*}{\textbf{Std}} \\\cmidrule{2-4}
		&\textbf{1} &\textbf{2} &\textbf{3} & & \\\midrule
		alien &1189.0 &825.0 &964.0 &992.7 &183.7 \\
		amidar &241.8 &165.3 &235.6 &214.2 &42.5 \\
		assault &618.4 &579.5 &585.8 &594.6 &20.9 \\
		asterix &665.0 &590.0 &495.0 &583.3 &85.2 \\
		bank\_heist &52.0 &462.0 &18.0 &177.3 &247.1 \\
		battle\_zone &17100.0 &20100.0 &21700.0 &19633.3 &2335.2 \\
		boxing &1.9 &6.1 &14.9 &7.6 &6.6 \\
		breakout &4.1 &6.1 &2.6 &4.3 &1.8 \\
		chopper\_command &1080.0 &950.0 &1240.0 &1090.0 &145.3 \\
		crazy\_climber &31200.0 &26170.0 &3400.0 &20256.7 &14813.4 \\
		demon\_attack &462.5 &479.5 &355.5 &432.5 &67.2 \\
		freeway &29.0 &30.3 &27.9 &29.1 &1.2 \\
		frostbite &355.0 &2484.0 &507.0 &1115.3 &1187.7 \\
		gopher &448.0 &338.0 &388.0 &391.3 &55.1 \\
		hero &7459.0 &7184.5 &7356.5 &7333.3 &138.7 \\
		jamesbond &425.0 &245.0 &480.0 &383.3 &122.9 \\
		kangaroo &660.0 &600.0 &1520.0 &926.7 &514.7 \\
		krull &3021.4 &3072.3 &3526.8 &3206.8 &278.3 \\
		kung\_fu\_master &23890.0 &33540.0 &4000.0 &20476.7 &15062.9 \\
		ms\_pacman &597.0 &1071.0 &1042.0 &903.3 &265.7 \\
		pong &-16.1 &-17.7 &-15.3 &-16.4 &1.2 \\
		private\_eye &100.0 &0.0 &100.0 &66.7 &57.7 \\
		qbert &2747.5 &870.0 &727.5 &1448.3 &1127.4 \\
		road\_runner &4510.0 &11550.0 &7210.0 &7756.7 &3551.7 \\
		seaquest &298.0 &744.0 &398.0 &480.0 &234.0 \\
		up\_n\_down &2678.0 &3451.0 &2871.0 &3000.0 &402.3 \\
		\bottomrule
	\end{tabular}
\end{table}

\begin{table}[!htp]\centering
	\caption{Auxiliary Tasks. Results are averaged across 3 runs.}\label{tab:detail-Aux}
	\scriptsize
	\begin{tabular}{lrrrrrr}\toprule
		&\textbf{NS} &\textbf{ID} &\textbf{RF} &\textbf{CI} &\textbf{CM} \\\midrule
		alien &1003.0 &1070.0 &824.0 &914.3 &820.3 \\
		amidar &129.5 &111.8 &142.0 &128.1 &94.2 \\
		assault &586.5 &502.2 &594.5 &480.2 &538.3 \\
		asterix &606.7 &506.7 &463.3 &548.3 &671.7 \\
		bank\_heist &108.7 &76.7 &91.0 &59.3 &90.3 \\
		battle\_zone &15000.0 &11433.3 &15400.0 &17700.0 &16766.7 \\
		boxing &-0.6 &1.8 &5.0 &3.1 &-0.3 \\
		breakout &3.2 &5.2 &3.9 &2.9 &3.9 \\
		chopper\_command &743.3 &743.3 &896.7 &723.3 &826.7 \\
		crazy\_climber &14273.3 &15073.3 &23663.3 &23933.3 &19923.3 \\
		demon\_attack &561.3 &577.7 &401.8 &461.7 &517.5 \\
		freeway &26.9 &27.3 &26.4 &27.3 &27.7 \\
		frostbite &432.7 &357.0 &480.7 &889.0 &779.0 \\
		gopher &342.7 &318.0 &376.7 &409.3 &369.3 \\
		hero &5639.5 &5522.3 &5641.2 &3980.7 &4770.2 \\
		jamesbond &310.0 &250.0 &380.0 &293.3 &351.7 \\
		kangaroo &493.3 &206.7 &466.7 &533.3 &526.7 \\
		krull &3209.0 &3240.7 &3422.1 &2735.6 &2623.4 \\
		kung\_fu\_master &9216.7 &6850.0 &13670.0 &6893.3 &11923.3 \\
		ms\_pacman &1405.0 &1036.7 &1044.7 &1820.7 &1236.0 \\
		pong &-18.3 &-9.0 &-11.7 &-13.9 &-16.3 \\
		private\_eye &-47.4 &-29.3 &66.7 &-276.2 &-39.6 \\
		qbert &935.8 &440.8 &970.0 &876.7 &509.2 \\
		road\_runner &12023.3 &14050.0 &12443.3 &13460.0 &11703.3 \\
		seaquest &517.3 &552.0 &622.0 &516.7 &503.3 \\
		up\_n\_down &3505.3 &3239.3 &5118.3 &3325.7 &3068.3 \\
		\bottomrule
	\end{tabular}
\end{table}
	
\end{document}